\documentclass{article}

\usepackage{microtype}
\usepackage{graphicx}
\usepackage{subfigure}
\usepackage{booktabs} 
\usepackage{amsmath}
\usepackage{amssymb}
\usepackage{multirow}
\usepackage{threeparttable}
\usepackage{stfloats}

\usepackage{hyperref}

\usepackage[accepted]{icml2021}

\icmltitlerunning{Winograd Algorithm for AdderNet}

\begin{document}
\newtheorem{theorem}{Theorem}
\newenvironment{proof}{{\noindent\it Proof}\quad}{\hfill $\square$\par}

\twocolumn[
\icmltitle{Winograd Algorithm for AdderNet}



\icmlsetsymbol{equal}{*}

\begin{icmlauthorlist}
\icmlauthor{Wenshuo Li}{hw}
\icmlauthor{Hanting Chen}{hw,bj}
\icmlauthor{Mingqiang Huang}{cas}
\icmlauthor{Xinghao Chen}{hw}
\icmlauthor{Chunjing Xu}{hw}
\icmlauthor{Yunhe Wang}{hw}
\end{icmlauthorlist}

\icmlaffiliation{hw}{Noah's Ark Lab, Huawei Technologies.}
\icmlaffiliation{bj}{Peking University.}
\icmlaffiliation{cas}{Shenzhen Institutes of Advanced Technology, Chinese Academy of Sciences}

\icmlcorrespondingauthor{Wenshuo Li}{liwenshuo@huawei.com}
\icmlcorrespondingauthor{Yunhe Wang}{yunhe.wang@huawei.com}

\icmlkeywords{Machine Learning, ICML}

\vskip 0.3in
]



\printAffiliationsAndNotice{} 

\begin{abstract}
Adder neural network (AdderNet) is a new kind of deep model that replaces the original massive multiplications in convolutions by additions while preserving the high performance. Since the hardware complexity of additions is much lower than that of multiplications, the overall energy consumption is thus reduced significantly. To further optimize the hardware overhead of using AdderNet, this paper studies the winograd algorithm, which is a widely used fast algorithm for accelerating convolution and saving the computational costs. Unfortunately, the conventional Winograd algorithm cannot be directly applied to AdderNets since the distributive law in multiplication is not valid for the $\ell_1$-norm. Therefore, we replace the element-wise multiplication in the Winograd equation by additions and then develop a new set of transform matrixes that can enhance the representation ability of output features to maintain the performance. Moreover, we propose the $\ell_2$-to-$\ell_1$ training strategy to mitigate the negative impacts caused by formal inconsistency. Experimental results on both FPGA and benchmarks show that the new method can further reduce the energy consumption without affecting the accuracy of the original AdderNet.
\end{abstract}

\section{Introduction}
\label{sec:intro}
The effectiveness of deep neural networks has been well demonstrated in a large variety of machine learning problems. With the rapid development of the accessible datasets, learning theory and algorithms and the computing hardware, the performance of considerable computer vision tasks has been improved by these neural networks, especially convolutional neural networks (CNNs). \cite{krizhevsky2012imagenet} first applies the deep CNN on the large-scale image classification and a series of subsequent network architectures are explored for boosting the accuracy such as ResNet~\cite{he2016deep}, EfficientNet~\cite{tan2019efficientnet}, and GhostNet~\cite{ghostnet}. In addition, there is a great number of networks presented for addressing different task including object detection~\cite{tan2020efficientdet}, segmentation~\cite{tao2020hierarchical}, and low-level computer vision tasks~\cite{guo2019toward,zhang2018densely,ren2019progressive}. Although these models can obtain state-of-the-art performance, most of them require massive computations and cannot be easily used on portable devices such as microphones, robots and self-driving cars.

\begin{figure}[t]
\centering
\vspace{-5pt}
\includegraphics[width=3in]{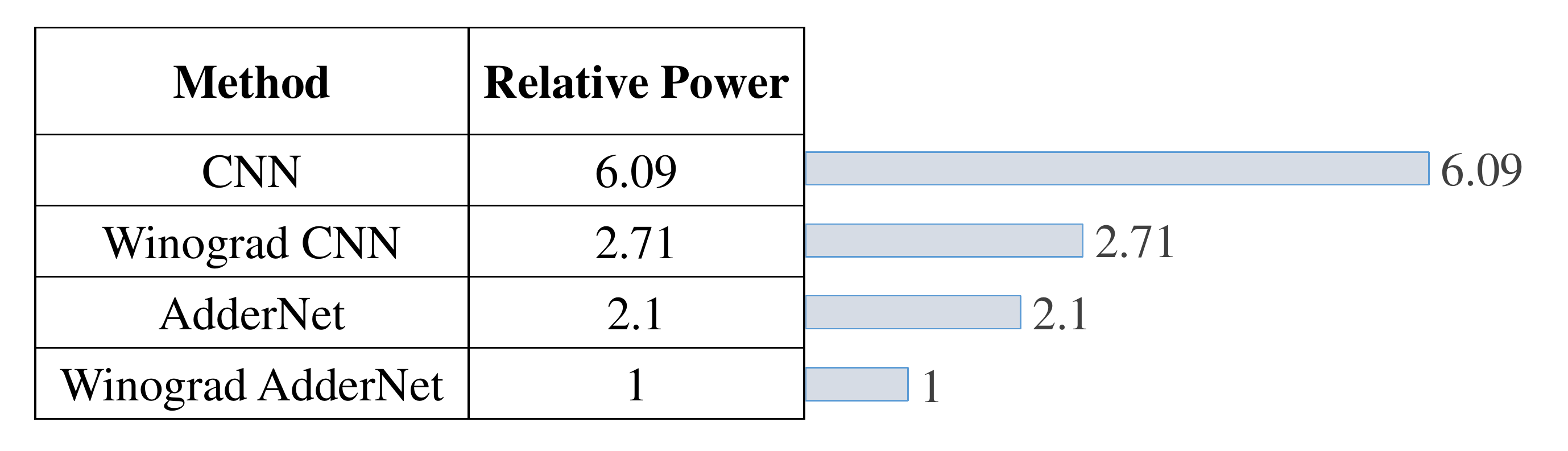}
\vspace{-10pt}
\caption{Comparison of relative power consuming between CNN, Winograd CNN, AdderNet and Winograd AdderNet. All data is achieved under 8-bit fixed-point number. *: The relative power of Winograd CNN is estimated by theoretical analysis.}
\vspace{-20pt}
\label{fig:power}
\end{figure}

To reduce the computational costs of pre-trained deep neural networks without affecting their performance is also a very important problem, a series of works have been explored for removing the network redundancy such as pruning~\cite{han2015learning}, distillation~\cite{hinton2015distilling}, and neural architecture search~\cite{liu2018darts}. Besides, according to the investigation in~\cite{dally2015high}, the energy consuming varies largely with different operations and different numeric precision (e.g., the energy consumption of an 32-bit multiplication is about 100$\times$ larger than that of an 8-bit addition). Therefore, quantization is now becoming the most common scheme for deploying deep neural networks on resource limited devices. \cite{qiu2016going} finds that 8-bit fixed-point number is sufficient for CNN to achieve a promising accuracy, and soon 8-bit fixed-point number becomes a common practice.

Furthermore, \cite{courbariaux2016binarized} proposes binary networks to quantize neural network to binary values (i.e., +1 and -1) to have an extreme simplification of deep networks. \cite{rastegari2016xnor} inserts a scale factor after each binarized layer to enhance the representational capability. \cite{lin2017towards} proposes ABC-Net to use the linear combination of binary bases to approximate floating-number weights and activations. \cite{liu2020reactnet} proposes RSign and RPReLU to learn the distribution reshape and shift for enhancing the performance. However, the main disadvantage of binary quantization is still the great loss of accuracy, e.g., the performance of the recent binary net is about 5\% lower than that of the baseline CNN with the same architecture on the ImageNet benchmark. Recently, \cite{addernet} proposes the adder neural network (AdderNet), which uses the conventional $\ell_1$-norm to calculate the output features. Since the cost of addition is much cheaper than that of multiplication (e.g., 8-bit addition is 7 times cheaper than 8-bit multiplication), AdderNet can significantly reduce the energy consumption of a given CNN model with comparable performance~\cite{xu2020kernel}.

Additionally, the fast calculation algorithms, including FFT~\cite{fft} and Winograd algorithm~\cite{winograd}, are widely used for improving the efficiency and reducing the computational complexity of deep neural networks. The Winograd algorithm is the most popular and effective method in acclerating CNNs~\cite{winograd-cnn}, since it has the best performance on accelerating $3\times 3$ layers, which is most commonly used in the modern neural architectures. Some following work focuses on the further optimization of the applications of Winograd algorithm for CNNs. To combine Winograd algorithm with neural network pruning, training in Winograd domain is proposed and the results show little loss of accuracy~\cite{liu2018efficient}.

Although the AdderNet can significantly reduce the overall energy cost of the resulting neural network, the energy consumption of convolutional layers could be obviously optimized by the Winograd algorithm as shown in Figure~\ref{fig:power}. Thus, we are motivated to explore the fast calculation algorithm for adder layers to further reduce the energy costs of using deep neural networks. However, due to distributive law is not applicable to the operation (i.e., sum of absolute values) in AdderNet, the conventional Winograd algorithm cannot be directly used. Therefore, we first thoroughly analyze the difficulties of applying the Winograd algorithm to AdderNet and explore a new paradigm for optimizing the inference of adder layers. The main contributions of this paper are summarized as follows:
\begin{itemize}
\vspace{-8pt}
\item We propose to inherit the original framework of the Winograd algorithm for optimizing AdderNet, and replace the original element-wise multiplication by adder operation, i.e., the $\ell_1$-norm for using additions.
\vspace{-5pt}
\item We then analyze the unbalance of feature maps in the Winograd for AdderNet, and investigate the optimal transform matrix for maximally enhancing the feature representation ability of the new output features. In addition, we present a $\ell_2$-to-$\ell_1$ training strategy to adapt the Winograd AdderNet paradigm and avoid the decline on the network performance. 
\vspace{-5pt}
\item Experiments conducted on benchmark datasets show that the performance of Winograd AdderNet is comparable to that of the baseline model, while achieving an about $2.1\times$ lower energy consumption on Field-Programmable Gate Array (FPGA).
\vspace{-5pt}
\end{itemize}


\section{Preliminaries}
\label{sec:related-work}
We briefly review the AdderNets and Winograd algorithm.
\subsection{AdderNet}
\label{ssec:addernet}
Different from convolutional neural networks, AdderNet~\cite{addernet} proposes to use $\ell_1$-norm to conduct the feed-forward process for extracting features. This method replaces multiplications with additions, which brings benefits for energy consumption and circuits area. The inference process is formulated as
\begin{equation}
\label{eq:adder}
\vspace{-3pt}
Y(m,n,t)=-\sum_{i,j,k}|F(i,j,k,t)-X(m+i,n+j,k)|,
\end{equation}
where $Y$ represents the output features, $F$ represents weights and $X$ represents input features. The backward process of weights $F$ and feature maps $X$ is approximated with $\ell_2$-norm and HardTanh instead of sign function, respectively.
\begin{equation}
\frac{\partial Y(m,n,t)}{\partial F(i,j,k,t)} = X(m+i,n+j,k)-F(i,j,k,t),
\end{equation}
\begin{equation}
\small \frac{\partial Y(m,n,t)}{\partial X(m+i,n+j,k)} = HT(F(i,j,k,t)-X(m+i,n+j,k)),
\end{equation}
where HT($\cdot$) is short for HardTanh function
$$
HT(\cdot)=
\begin{cases}
x,& -1 < x <1,\\
-1,& x < -1,\\
1,& x > 1.
\end{cases}
$$
Since the norms of gradients in AdderNet are smaller than that in CNNs, the authors propose an adaptive learning rate for different layers in AdderNets. The learning rate for each layer $l$ could be formulated by:
\begin{equation}
\Delta F_l =\gamma \times \alpha_l \times \Delta L(F_l),
\end{equation}
\begin{equation}
\label{eq:alr}
\alpha_l=\frac{\eta \sqrt{k}}{||\Delta L(F_l)||_2}.
\end{equation}
The following work expands the application scope, such as super-resolution~\cite{addersr}. AdderNet has shown its potential to replace CNN in many computer vision tasks and attracted a lot of attention.
\subsection{Winograd algorithm}
\label{ssec:winograd}
Winograd algorithm~\cite{winograd} is a widely used fast calculation method, which can accelerate the convolution calculation in the signal processing area. \cite{winograd-cnn} applies Winograd algorithm in convolutional neural networks and largely reduce the computation cost of CNNs. Denote the length of filter as $r$, the output length as $m$ and the corresponding Winograd algorithm as $F(m, r)$.

The Winograd algorithm of $F(2,3)$ can be formulated as
\begin{equation}
\label{eq:wino}
Y=A^T[[GgG^T]\odot[B^TdB]]A,
\end{equation}
\begin{equation*}
A = \left[
 \begin{matrix}
   1 & 0\\
   1 & 1\\
   1 & -1\\
   0 & -1
  \end{matrix}
  \right],
G = \left[
 \begin{matrix}
   1 & 0 & 0 \\
   \frac{1}{2} & \frac{1}{2} & \frac{1}{2} \\
   \frac{1}{2} & -\frac{1}{2} & \frac{1}{2} \\
   0 & 0 & 1
  \end{matrix}
  \right],
\end{equation*}
\begin{equation}
\label{eq:wino3}
B = \left[
 \begin{matrix}
   1 & 0 & 0 & 0 \\
   0 & 1 & -1 & 1 \\
   -1 & 1 & 1 & 0 \\
   0 & 0 & 0 & -1
  \end{matrix}
  \right].
\end{equation}
where $g$ represents the $3\times 3$ convolution filter, $d$ represents a $4\times 4$ tile of input feature map and $\odot$ respresents the element-wise multiplication. $A$, $G$ and $B$ are the transform matrix for output, weights and input, respectively. When forward process is performed, $\hat{g}=GgG^T$ could be calculated at advance to reduce the calculation overhead, since $g$ would not be changed. Therefore, Equation~(\ref{eq:wino}) can be reformulated as:
\begin{equation}
\label{eq:hatg}
\vspace{-3pt}
Y=A^T[\hat{g}\odot[B^TdB]]A.
\end{equation}
Considering the complexity of transformation and the numeric precision issues, $F(2\times 2,3\times 3)$ is the most commonly used form in practice~\cite{liu2018efficient, yan2020optimizing}. Moreover, with $m>2$ or $r>3$, the transformation matrix could not be binary, making it harder to be applied to AdderNets. Therefore, we only focus on $F(2\times 2,3\times 3)$ in the following sections.

\section{Method}
\subsection{Winograd Algorithm for AdderNet}
\label{sec:basic-form}
As AdderNet and Winograd algorithm can all improve the efficiency of neual networks, we explore to combine the two techniques together to further reduce the computation cost.

In this section, we will introduce the vanilla form of Winograd algorithm on AdderNet. As shown in Equation~(\ref{eq:hatg}), the Winograd algorithm consists of several parts of calculations, including pre-transformations for filters $GgG^T$, pre-transformations for inputs $B^TdB$, element-wise multiplications, and output transformations using matrix $A$. Since the calculations in input pre-transformations and output transformations only contain additions, we do not need to modify them. Therefore, we only replace the element-wise multiplications with $\ell_1$-distance, the calculations can be reformulated as:
\begin{equation}
\label{eq:winoadder}
Y=A^T[-|\hat{g}\ominus[B^TdB]|]A.
\end{equation}
$A$, $G$, and $B$ are the same as those we introduced in Section~\ref{ssec:winograd}. $\ominus$ represents element-wise minus operation. (Additions and minus operation are actually the same, since minus operation could be implemented by additions of complement.) And $|\cdot|$ represents the absolute operation for each element in the matrix.

Here we give a brief analysis of the complexity of Winograd algorithm for AdderNet. The calculation of Winograd algorithm consists of four parts: weight pre-transformations, input pre-transformations, element-wise minus and absolute operation of weights and output transformations. The transformation of weights could be calculated before deployment, so we do not take this part into account. We denote the shape of input features as $(N, C_{in}, X_h, X_w)$, and the shape of weights as $(C_{out}, C_{in}, K_h, K_w)$. The input features could be divided into $N\times C_{in}\times \frac{X_h}{2}\times \frac{X_w}{2}$ groups for applying Winograd algorithm. Each group requires 3 additions since each column and row of matrix $B$ has two non-zero values ($1$ or $-1$), which means the final $B^TdB$ results are the sum of four values. For the element-wise minus and absolute operation of weights and features, each group requires $16$ additions, and there are $N\times C_{out}\times C_{in}\times \frac{X_h}{2}\times \frac{X_w}{2}$ groups in total. Since the results of addition and absolute operation need to be accumulated, the times of additions should be doubled, which results in $N\times C_{out}\times C_{in}\times \frac{X_h}{2}\times \frac{X_w}{2}\times 16\times 2$ additions. The output features could be divided into $C_{out}\times X_h\times X_w$ groups and each group needs 8 additions since the matrix $A$ has 3 non-zero values each column. The total additions of three parts is
\begin{equation}
\label{eq:comp}
N\times \frac{X_h}{2}\times \frac{X_w}{2}\times (C_{out}\times C_{in}\times 16\times 2 + C_{in}\times 3 + C_{out}\times 8).
\end{equation}
Since the values of $C_{in}$ and $C_{out}$ are generally dozens or hundreds, last two items can be ignored. Then the formula~(\ref{eq:comp}) becomes
\vspace{-5pt}
\begin{equation}
N\times X_h\times X_w\times C_{out}\times C_{in}\times 8.
\end{equation}
The total additions of original AdderNet are
\vspace{-5pt}
\begin{equation}
N\times X_h\times X_w\times C_{in}\times C_{out}\times 9\times 2.
\end{equation}
Thus, the Winograd algorithm for AdderNet only requires about $\frac{4}{9}$ additions of original AdderNet.

However, since the absolute value is used in the calculation of AdderNet, the distributive law is not valid. So the Winograd form in Equation~(\ref{eq:winoadder}) is not equal to the original addition operation in Equation~(\ref{eq:adder}). Moreover, the accumulative absolute values put forward higher requirements for output transform matrix $A$. Let $X=-|\hat{g}\ominus[B^TdB]|$, and denote elements in $X$ and $Y$ as
$$
X = \left[
 \begin{matrix}
   x_0 & x_1 & x_2 & x_3 \\
   x_4 & x_5 & x_6 & x_7 \\
   x_8 & x_9 & x_{10} & x_{11} \\
   x_{12} & x_{13} & x_{14} & x_{15} \\
  \end{matrix}
  \right],
Y = \left[
 \begin{matrix}
   y_0 & y_1 \\
   y_2 & y_3 \\
  \end{matrix}
  \right].
$$
We expand the equation $Y=A^TXA$ and get
\begin{align*}
\vspace{-8pt}
y_0 &= x_0+x_1+x_2+x_4+x_5+x_6+x_8+x_9+x_{10},\\
y_1 &= x_1-x_2-x_3+x_5-x_6-x_7+x_9-x_{10}-x_{11},\\
y_2 &= x_4+x_5+x_6-x_8-x_9-x_{10}-x_{12}-x_{13}-x_{14},\\
y_3 &= x_5-x_6-x_7-x_9+x_{10}+x_{11}-x_{13}+x_{14}+x_{15}.
\vspace{-8pt}
\end{align*}
We can find that the number of additions and that of minus operations in each equation is not the same. The magnitude of each x is usually similar. Since all elements in X are negative, the magnitude of output features $Y$ is not consistent for each $y_i$. The unbalance of different positions obviously affect the performance of the network. In the next section, we will introduce our method to mitigate the influence of this two problem.

\subsection{Optimal Transform Matrix}
\label{sec:mod-of-trans}
As discussed above, if we apply the Winograd algorithm for AdderNets, the overall computational complexity during the inference can be reduced. In this section, we explore new transform matrixes to solve the feature unbalanced problem. There are two requirements of the transform matrixes.
\begin{itemize}
\vspace{-6pt}
\item The modified transform matrixes could balance the magnitude of all positions of output features $Y$.
\vspace{-2pt}
\item The output of Winograd algorithm transformed by modified matrixes should be equal to that of the original form in CNN.
\vspace{-6pt}
\end{itemize}
The first requirement is to ensure the output of Winograd AdderNet have a similar magnitude which can be properly handled by the following layers (Batchnorm and ReLU). The second requirement is to remain the basic characteristic of conventional Winograd algorithm and make it universal to CNNs. 
\begin{theorem}
The general solution of the Winograd form $F(2, 3)$ is
$$
Y=A^T[[GgG^T]\odot[B^TdB]]A.
$$
and the matrixes are
$$
A = \left[
 \begin{matrix}
   \alpha_0 & -\alpha_0 c_0\\
   \beta_0 & -\beta_0 c_1\\
   \gamma_0 & -\gamma_0 c_2 \\
   0 & \delta_0
  \end{matrix}
  \right],
$$
$$
G = \left[
 \begin{matrix}
   \frac{\alpha_1}{(c_1-c_0)(c_2-c_0)} & -\frac{\alpha_1c_0}{(c_1-c_0)(c_2-c_0)} & \frac{\alpha_1c_0^2}{(c_1-c_0)(c_2-c_0)} \\
   \frac{\beta_1}{(c_0-c_1)(c_2-c_1)} & -\frac{\beta_1c_1}{(c_0-c_1)(c_2-c_1)} & \frac{\beta_1c_1^2}{(c_0-c_1)(c_2-c_1)} \\
   \frac{\gamma_1}{(c_0-c_2)(c_1-c_2)} & -\frac{\gamma_1c_2}{(c_0-c_2)(c_1-c_2)} & \frac{\gamma_1c_2^2}{(c_0-c_2)(c_1-c_2)} \\
   0 & 0 & \delta_1
  \end{matrix}
  \right],
$$
$$
B = \left[
 \begin{matrix}
   \frac{c_1c_2}{\alpha_0\alpha_1} & \frac{c_0c_2}{\beta_0\beta_1} & \frac{c_0c_1}{\gamma_0\gamma_1} & \frac{c_0c_1c_2}{\delta_0\delta_1} \\
   \frac{c_1+c_2}{\alpha_0\alpha_1} & \frac{c_0+c_2}{\beta_0\beta_1} & \frac{c_0+c_1}{\gamma_0\gamma_1} & \frac{c_0c_1+c_0c_2+c_1c_2}{\delta_0\delta_1} \\
   \frac{1}{\alpha_0\alpha_1} & \frac{1}{\beta_0\beta_1} & \frac{1}{\gamma_0\gamma_1} & \frac{c_0+c_1+c_2}{\delta_0\delta_1} \\
   0 & 0 & 0 & \frac{1}{\delta_0\delta_1}
  \end{matrix}
  \right].
$$
in which $c_0$, $c_1$ and $c_2$ are arbitrary rational numbers, and $\alpha_i$, $\beta_i$, $\gamma_i$ and $\delta_i$, $i=0,1$ are arbitrary real numbers.
\end{theorem}
\begin{proof}
For the Winograd form F(2, 3) , denote input sequence $y$ and filter sequence $g$ with length two and three as $[y_0, y_1]$ and $[g_0, g_1, g_2]$, then the results of convolution operation $[d_0, d_1, d_2, d_3]$ would be
\vspace{-5pt}
\begin{align*}
(y*g)(\tau)&=\sum_{t,\tau -t>0}^3 y_tg_{\tau-t},\\
d_0=(y*g)(0)&=y_0g_0, d_1=(y*g)(1)=y_0g_1+y_1g_0,\\
d_2=(y*g)(2)&=y_1g_1+y_2g_0, d_3=(y*g)(3)=y_1g_2.
\end{align*}

The results of the convolution can be derived from the product of two discrete sequence polynomials $y(n)$ and $g(n)$.
\vspace{-10pt}
\begin{align*}
&y(n)=y_0+y_1n,g(n)=g_0+g_1n+g_2n^2,\\
&d(n)=y(n)g(n)\\
&=y_0g_0+(y_0g_1+y_1g_0)n+(y_0g_2+y_1g_1)n^2+y_1g_2n^3\\
&=(y*g)(0)+(y*g)(1)n+(y*g)(2)n^2+(y*g)(3)n^3.
\end{align*}
The coefficients of $n^{i}, i=0...3$ term in polynomials $d(n)$ are actually the results of the $i$th term in convolution, so we can get the results of convolution by calculating the polynomials. In order to solve the polynomial coefficients, we need to construct an equivalent transformation for $d(n)$. We divide $d(n)$ into two parts, mutual prime polynomial $M(n)$ and remainder $d'(n)$. The order of $M(n)$ is the same as that of $d(n)$ so that the coefficient of $M(n)$ is  Then the problem of solving polynomial coefficients is converted into the problem of solving remainders. Denote three relatively prime polynomials as $m_0(n)=a_0n+b_0$, $m_1(n)=a_1n+b_1$, $m_2(n)=a_2n+b_2$, in which $a_0, a_1, a_2$ are arbitrary non-zero integers and $b_0, b_1, b_2$ are arbitrary integers. Then we get
\vspace{-6pt}
\begin{align}
M(n) &= (a_0n+b_0)(a_1n+b_1)(a_2n+b_2)\\
&= a_0a_1a_2(n+c_0)(n+c_1)(n+c_2)\\
c_0&=b_0/a_0, c_1=b_1/a_1, c_2=b_2/a_2.
\end{align}
where $c_0$, $c_1$ and $c_2$ are arbitrary rational numbers, and
\begin{equation}
\label{eq:svalue}
d(n)=tM(n)+d'(n), t=y_1g_2.
\end{equation}
We use Extended Euclidean algorithm to solve the inverse elements $[(\frac{M(n)}{m_i(n)})^{-1}]_{m_i(n)}$. The results are $\frac{1}{(c_1-c_0)(c_2-c_0)}$, $\frac{1}{(c_0-c_1)(c_2-c_1)}$, $\frac{1}{(c_0-c_2)(c_1-c_2)}$ for $m_0(n)$, $m_1(n)$, $m_2(n)$, respectively. According to Chinese remainder theorem,
\begin{equation}
d'(n)=\sum_{i=0}^2d'_i(n)\frac{M(n)}{m_i(n)}[(\frac{M(n)}{m_i(n)})^{-1}]_{m_i(n)}.
\end{equation}
Substituted into Equation~(\ref{eq:svalue}) and then we get
\begin{align}
\label{eq:s}
d_0'(n)&=\frac{(y_0-c_0y_1)(c_0^2g_2-c_0g_1+g_0)}{(c_1-c_0)(c_2-c_0)},\\
d_1'(n)&=\frac{(y_0-c_1y_1)(c_1^2g_2-c_1g_1+g_0)}{(c_0-c_1)(c_2-c_1)},\\
d_2'(n)&=\frac{(y_0-c_2y_1)(c_2^2g_2-c_2g_1+g_0)}{(c_1-c_2)(c_0-c_2)},\\
t&=y_1g_2.
\label{eq:s1}
\end{align}
and the input transformation
\begin{align*}
&d(n)=tn^3+[d_0'(n)+d_1'(n)+d_2'(n)+(c_0+c_1+c_2)t]n^2\\
&+[(c_1+c_2)d_0'(n)+(c_0+c_2)d_1'(n)+(c_0+c_1)d_2'(n)\\
&+(c_0c_1+c_0c_2+c_1c_2)t]n\\
&+(c_1c_2d_0'(n)+c_0c_2d_1'(n)+c_0c_1d_2'(n)+c_0c_1c_2t].
\end{align*}
From Equation~(\ref{eq:s})-(\ref{eq:s1}), we have
$$
\left[
\begin{matrix}
\alpha_0\alpha_1d_0'(n) \\
\beta_0\beta_1d_1'(n) \\
\gamma_0\gamma_1d_2'(n) \\
\delta_0\delta_1t
\end{matrix}
\right] = A \cdot
\left[
\begin{matrix}
y_0 \\
y_1
\end{matrix}
\right]
\odot G \cdot
\left[
\begin{matrix}
g_0 \\
g_1 \\
g_2
\end{matrix}
\right],
A = \left[
 \begin{matrix}
   \alpha_0 & -\alpha_0 c_0\\
   \beta_0 & -\beta_0 c_1\\
   \gamma_0 & -\gamma_0 c_2 \\
   0 & \delta_0
  \end{matrix}
  \right].
$$
Since the division of $d_0'(n)$, $d_1'(n)$, $d_2'(n)$and $t$ is arbitrary, we add coefficients $\alpha_i$, $\beta_i$, $\gamma_i$ and $\delta_i$, $i=0,1$ to maintain the generability of the solution. Moreover, the order of rows in matrix $A$ and matrix $G$ can be swapped simultaneously, and the corresponding column in matrix $B$ should also be swapped. Then the coefficients of polynomial $s(x)$ can be represented with$[\alpha_0\alpha_1d_0'(n), \beta_0\beta_1d_1'(n), \gamma_0\gamma_1d_2'(n), \delta_0\delta_1t]$
$$
\left[
\begin{matrix}
d_0\\
d_1\\
d_2\\
d_3
\end{matrix}
\right] = B \cdot \left[
\begin{matrix}
\alpha_0\alpha_1d_0'(n)\\
\beta_0\beta_1d_1'(n)\\
\gamma_0\gamma_1d_2'(n)\\
\delta_0\delta_1t
\end{matrix}
\right].
$$
Now we get the convolution result $x = B[Gg\odot Ay]$. For the correlation operation we need in CNN, $s$ would be the input and $h$ would be the output, so we have $y = A^T[Gg \odot B^Td]$. Then the 1-D result is nested to itself to obtain the 2-D result $Y = A^T[GgG^T\odot B^TdB]A$.
\end{proof}
In order to reduce the amount of calculation during the inference process, the elements in matrix $A$ should be chosen from $0$, $1$, $-1$ to avoid shift or multiplication operations. So $c_0$, $c_1$ and $c_2$ could only be chosen from $0$, $-1$ and $1$, one of each. Without losing generality, we set $\alpha_1=-1$, $\delta_0=-1$ and other coefficients to $1$, then we get the standard Winograd algorithm for convolution like Equation~(\ref{eq:wino})-(\ref{eq:wino3}).

Denote that the number of $+1$ and $-1$ in matrix $A$ of column $i$ is $p_i$ and $k-p_i$, in which $k$ represents the total number of non-zero elements. According to the previous results, we have $k=3$ in all columns of matrix $A$.
\begin{theorem}
$\forall i,j,m,n$, the number of additions and minus operations in the calculations of output feature $Y_{i,j}$ and $Y_{m,n}$ would be equal respectively if and only if $p_i=p_j=p_m=p_n$.
\end{theorem}
\begin{proof}
$\forall i,j,m,n$, if the number of additions of $Y_{i,j}$ and $Y_{m,n}$ is the same, then we have
$$p_ip_j+(k-p_i)(k-p_j)=p_mp_n+(k-p_m)(k-p_n)$$
$$k[(p_i+p_j)-(p_m+p_n)]=2(p_ip_j-p_mp_n)$$
Since $i,j,m,n$ are all arbitrary, to ensure the equation always established, we have $(p_i+p_j)-(p_m+p_n)\equiv 0$ and $p_ip_j - p_mp_n\equiv 0$. That is to say, $p_i=p_m,p_j=p_n$ or $p_i=p_n,p_j=p_m$. Since $i,j,m,n$ are arbitrary, actually we have $p_i=p_j=p_m=p_n$.
\end{proof}
Based on the conclusions we deduce above, we can modify the coefficients $\alpha_i$, $\beta_i$, $\gamma_i$ and $\delta_i$, $i=0,1$ to let the number of $+1$ and $-1$ in every column of matrix $A$ keep the same. It is easy to draw the conclusion that there are only four matrixes $A_i, i=0...3$ which meets our requirements.
$$
A_0^T = \left[
 \begin{matrix}
   -1 & 1 & 1 & 0\\
   0 & 1 & -1 & 1
  \end{matrix}
  \right],
A_1^T = \left[
 \begin{matrix}
   -1 &-1 & 1 &0\\
   0 & -1 & -1 & 1
  \end{matrix}
  \right],
$$
$$
A_2^T = \left[
 \begin{matrix}
   1 & -1 & -1 & 0\\
   0 & -1 & 1 & -1
  \end{matrix}
  \right],
A_3^T = \left[
 \begin{matrix}
   1 & 1 & -1 & 0\\
   0 & 1 & 1 & -1
  \end{matrix}
  \right].
$$
Correspondingly, we can get the matrixes $G_i, i=0...3$.

\begin{table*}[h]
\centering
\caption{Results on CIFAR-10 and CIFAR-100 datasets}
\label{table:cifar}
\begin{tabular}{cccccc}
\hline
Model                      & Method    & \#Mul & \#Add & CIFAR-10 Accuracy & CIFAR-100 Accuracy \\ \hline
\multirow{3}{*}{ResNet-20} & Winograd CNN & 19.40M & 19.84M & 92.25\%  & 68.14\% \\ \cline{2-6}
					& AdderNet   & - & 80.74M     &   91.84\%       &    67.60\%       \\ \cline{2-6}
                           & Winograd AdderNet & - & 39.24M  &   91.56\%       &    67.96\%       \\ \hline
\multirow{3}{*}{ResNet-32}  & Winograd CNN & 31.98M & 32.74M & 93.29\% & 69.74\% \\ \cline{2-6}
					& AdderNet   & - & 137.36M     &    93.01\%      &    69.02\%        \\ \cline{2-6}
                           & Winograd AdderNet & - & 64.72M  &   92.34\%      &    69.87\%       \\ \hline
\end{tabular}
\vspace{-10pt}
\end{table*}
\subsection{Training with L2-to-L1 Distance}
\label{sec:training}
Although we can solve the feature unbalanced problem by optimal transform matrix proposed in the above section, the winograd form of AdderNet is still not equal to its original form. Here we introduce a training method to mitigate the gap of Winograd AdderNet and original AdderNet. According to the training strategy described in AdderNet~\cite{addernet}, the output of $\ell_2$-AdderNet can be regarded as a linear transformation of that in the conventional CNN. For Winograd algorithm, $\ell_2$-norm also means a better appoximation of multiplication, since $\ell_2$ is composed of squares and multications, i.e.,
\begin{align*}
Y&=A^T[-([GgG^T]\ominus[B^TdB])^2]A\\
&=A^T[(2[GgG^T]\odot[B^TdB]\ominus[GgG^T]^2\ominus[B^TdB]^2)]A
\end{align*}
However, $\ell_1$-norm is more hardware friendly than $\ell_2$-norm since it requires no multiplications. To improve the network accuracy and maintain the hardware efficiency, we propose to train AdderNet after applying the Winograd algorithm in an $\ell_2$-to-$\ell_1$ distance paradigm.

In the inference process, the adder layer is formulated as
\vspace{-10pt}
\begin{equation}
t=F(i,j,k,t)-X(m+i,n+j,k).
\end{equation}
\begin{equation}
Y(m,n,t)=-\sum_{i,j,k}(|t|)^p.
\end{equation}
The backward process is formulated as
\vspace{-10pt}
\begin{equation}
\frac{\partial Y(m,n,t)}{\partial X(m+i,n+j,k)} = p\cdot t^{p-1}\cdot sign(t).
\end{equation}
\begin{equation}
\frac{\partial Y(m,n,t)}{\partial F(i,j,k,t)} = p\cdot (-t)^{p-1}\cdot sign(-t).
\end{equation}
During the training process, we gradually reduce the exponent $p$ from 2 to 1. Then the forward and backward process finally calculated as followed,
\vspace{-10pt}
\begin{equation}
Y(m,n,t)=-\sum_{i,j,k}|t|
\end{equation}
\vspace{-10pt}
\begin{equation}
\frac{\partial Y(m,n,t)}{\partial X(m+i,n+j,k)} = sign(t).
\end{equation}
\begin{equation}
\frac{\partial Y(m,n,t)}{\partial F(i,j,k,t)} = sign(-t).
\end{equation}
To ensure the continuity of approximation, we do not apply the $\ell_2$ gradients for $F$ and HardTanh gradients for $X$ as AdderNet. The adaptive learning rate in Equation~(\ref{eq:alr}) proposed in AdderNet is adapted to stabilize the training process.

There are several strategies to reduce the exponent $p$. We denote the step of reduction as $s$.
\begin{itemize}
\vspace{-8pt}
\item \textbf{Training until converge and then reducing $p$.} Train network with cosine annealing learning rate until the learning rate close to 0. Then reduce $p$ with a certain step $s$ and restart the training process.
\vspace{-6pt}
\item \textbf{Reducing $p$ during the converge process.} Reduce $p$ every $k$ epoch of the training process and the step $s$ is set to $\frac{k}{epochs}$.
\vspace{-5pt}
\end{itemize}
We will give detailed analysis of different strategies in the experimental results.

Since the kernel transformation $\hat{g}=GgG^T$ in Winograd algorithm is not equivalent for AdderNet, we do not perform weight transformation during the training process. Instead, we directly train the weights in the Winograd domain. We also compare different ways to deal with weights in the ablation study part.
\begin{figure*}
\vspace{-8pt}
	\centering
	\begin{minipage}[b]{0.3\textwidth}
\centering
		\includegraphics[width=0.94\textwidth]{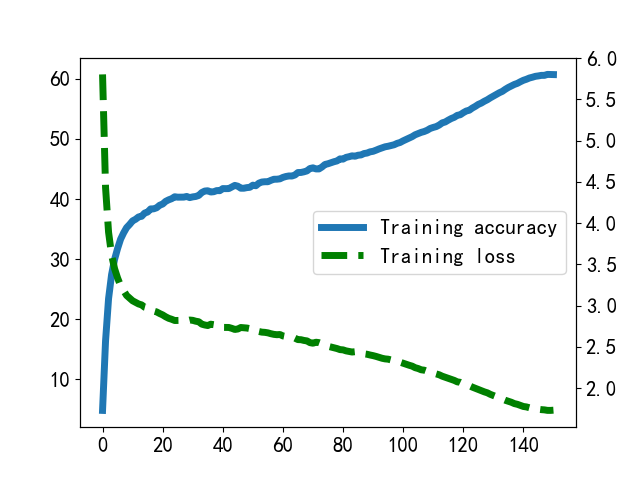} 
\label{fig:res18}
\vspace{-10pt}
\caption{Training accuracy and training loss of Winograd AdderNet ResNet-18 on ImageNet}
	\end{minipage}
\hspace{3pt}
	\begin{minipage}[b]{0.6\textwidth}
\centering
		\subfigure{
			\includegraphics[width=0.47\textwidth]{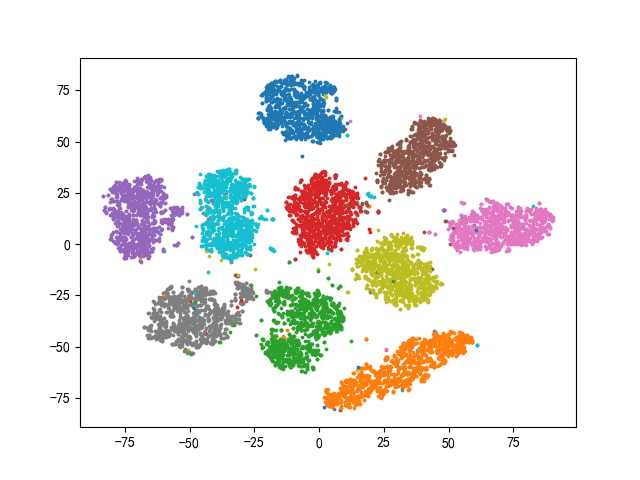}
			\includegraphics[width=0.47\textwidth]{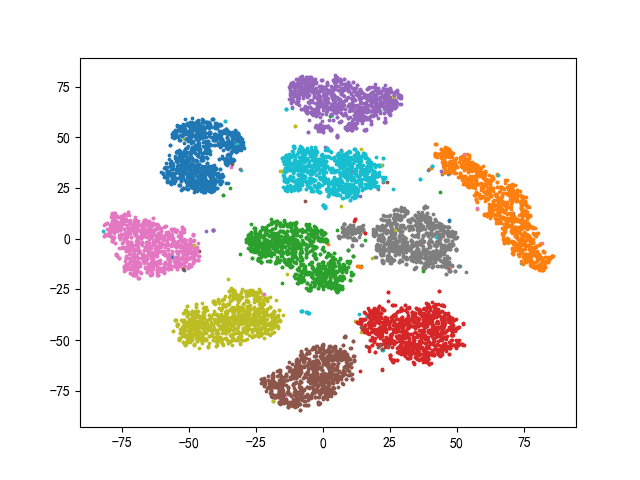}
			\label{fig:vis}
	}
\vspace{-10pt}
\caption{Dimension reduction results of features in Winograd for AdderNet (left) and original AdderNet (right). The visualization results are very close, which means that Winograd for AdderNet attracts the similar features to original AdderNet}
	\end{minipage}
\end{figure*}
\begin{figure*}[h]
\centering
\vspace{-13pt}
\subfigure[Input features]{
\label{fig:fea1}
\includegraphics[width=0.26\textwidth]{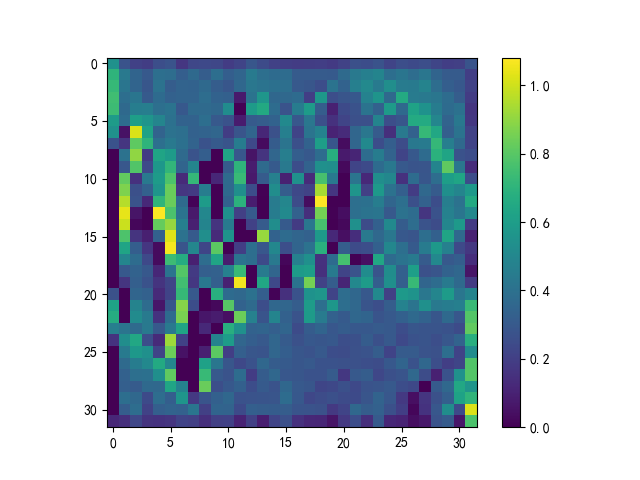}}
\subfigure[Output features with modified $A$]{
\label{fig:fea2}
\includegraphics[width=0.26\textwidth]{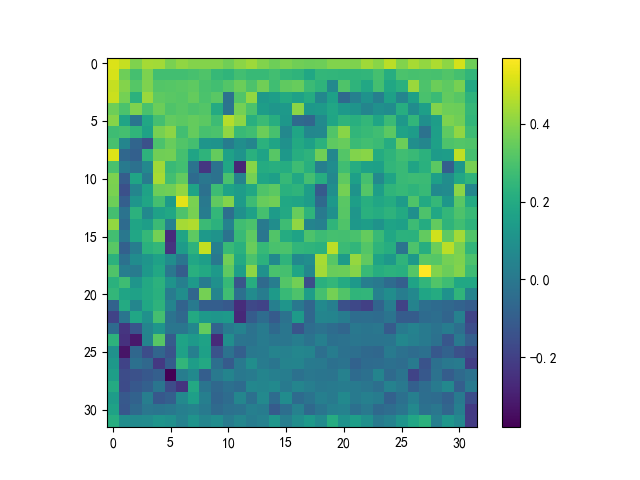}}
\subfigure[Output features with original $A$]{
\label{fig:fea3}
\includegraphics[width=0.26\textwidth]{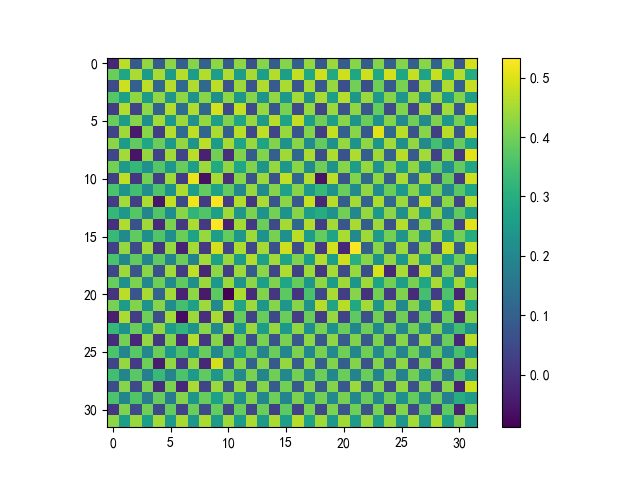}}
\vspace{-11pt}
\caption{Comparison of the feature heatmaps under different matrix $A$. Without the modified matrix $A$, there is a obvious grid in the heatmap shown in figure~\ref{fig:fea3}.}
\vspace{-11pt}
\label{fig:feature}
\end{figure*}

\vspace{-4pt}
\section{Experiments}
\label{sec:exp}
Here we conduct experiments to show the effectiveness of our proposed Winograd algorithm for AdderNet. The experiments are done on several commonly used datasets, including MNIST, CIFAR and ImageNet. We also make some ablation studies and give visualization of features to provide insights of our methods. All experiments are made via PyTorch on NVIDIA Tesla V100 GPU. For all experiments, we use the transform matrix $A_0$ and $G_0$, and other $A_i$ and $G_i$ matrixes can achieve the similar results.

\subsection{Classification}
\paragraph{Experiments on MNIST} First we evaluate our method on the MNIST dataset. To facilitate Winograd algorithm, we replace $5\times 5$ layers with $3\times 3$ layers in the original LeNet-5-BN~\cite{lecun1998gradient}. The detailed network structure is shown in the supplemental material. The learning rate is set to 0.1 at the beginning and decay with the cosine function in the following 100 epochs. We use SGD optimizer with momentum as 0.9, and the batch size is set as 256.

The original AdderNet achieves a $99.28\%$ accuracy while the Winograd AdderNet achieves $99.19\%$. Meanwhile, the Winograd AdderNet requires only 401.1M additions instead of 746.8M additions required by original AdderNet.
\vspace{-8pt}
\paragraph{Experiments on CIFAR}
We also evaluate Winograd AdderNet on the CIFAR dataset, including CIFAR-10 and CIFAR-100. The data settings are the same as~\cite{he2016deep}. The initial learning rate is set to 0.1 and then decays with a cosine learning rate schedule. The model is trained for 800 epochs and the training batch size is 256. The hyper-parameter $\eta$ in Equation~(\ref{eq:alr}) is set to 0.1.

To make fair comparison, we follow the settings in~\cite{addernet} to set the first and last layers as full-precision convolutional layers. The results are shown in table~\ref{table:cifar} and the results of CNN and AdderNet are from~\cite{addernet}. We only count the additions of adder part instead of the whole neural network. At the cost of little accuracy loss, the number of additions is reduced by more than 50\%.

\vspace{-12pt}
\paragraph{Experiments on ImageNet}
ImageNet is a large scale vision dataset which consists of $224\times 224$ pixel RGB images. We train ResNet-18 follow the original data settings in~\cite{he2016deep}. We train Winograd AdderNet on 8 GPUs with batch size 512, and the total training epochs are 150. The weight decay is set as $0.0001$ and the momentum is 0.9. The hyper-parameter $\eta$ is set to 0.05 for Winograd AdderNet. Experimental results are shown in Figure 2 and we use the AdderNet baseline from their paper~\cite{addernet}. Winograd AdderNet achieves a 66.2\% top-1 accuracy and an 86.8\% top-5 accuracy in ResNet-18, which is slightly less than AdderNet (67.0\% top-1/87.6\% top-5). If we extend the training epochs to 250, Winograd AdderNet achieved 66.5\% top-1 accuracy while AdderNet got no improvement. Besides, Winograd AdderNet uses only $1.72$G adder operations compared with $3.39$G in AdderNet.

\begin{table*}[ht]
\vspace{-15pt}
\centering
\begin{threeparttable}[b]
\caption{FPGA Simulation Results of original AdderNet and Winograd AdderNet}
\label{table:fpga}
\begin{tabular}{c|c|ccc}
\hline
Method                             & Module           & \#cycle    & Hardware Resource & Total Energy Consuming (Equivalent)\tnote{1} \\ \hline
original AdderNet                  & total      & 7062       & 7130              & 50.4M                               \\ \hline
\multirow{5}{*}{Winograd AdderNet} & padding          & 900        & 31                & 0.03M                               \\
                                   & input transform  & 3136       & 433               & 1.36M                               \\
                                   & calculation      & 3140       & 6900              & 21.7M                               \\
                                   & output transform & 3136       & 309               & 0.97M                               \\ \cline{2-5}
                                   & total            & -  & 7673              & \textbf{24.0M}                               \\ \hline
\end{tabular}
\begin{tablenotes}
     \item[1] \footnotesize Since the ratio of hardware resource usage is close to 100\%, we use the hardware resource overhead to approximate equivalent power consumption.
   \end{tablenotes}
  \end{threeparttable}
\vspace{-10pt}
\end{table*}

\begin{figure*}[h]
\centering
\vspace{-2pt}
\includegraphics[width=6in]{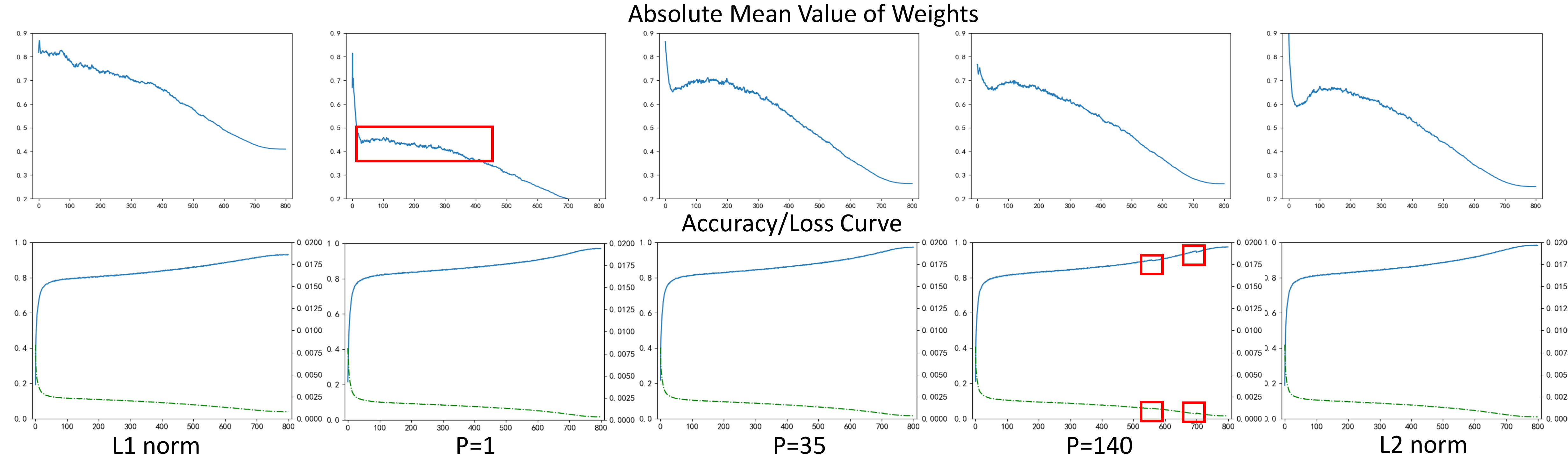}
\vspace{-10pt}
\caption{Upper: Trending of bsolute mean value of weights during training process. Lower: Training loss and accuracy.}
\vspace{-10pt}
\label{fig:training}
\end{figure*}

\begin{table}[htp]
\centering
\vspace{-13pt}
\caption{Ablation Study on the Reduction Method of $p$}
\label{table:reducep}
\begin{tabular}{c|c}
\hline
Method                  & Accuracy \\ \hline
Training until converge &   89.24       \\
Reducing during converge with $p=1$    &   90.94       \\
Reducing during converge with $p=35$       &  \textbf{91.56}        \\
Reducing during converge with $p=140$       &  91.44        \\ \hline
\end{tabular}
\vspace{-13pt}
\end{table}
\subsection{Ablation Study}
In this part, we evaluate the effectiveness of our proposed transform matrix and $\ell_2$-to-$\ell_1$ training method. All experiments in this section is done with ResNet-18 network on the CIFAR-10 dataset.

First we compare different methods to reduce parameter $p$ in the $\ell_2$-to-$\ell_1$ training strategy. The total training epochs are set to 800 to make fair comparison. The evaluation results are shown in Table~\ref{table:reducep}. We can find that reducing $p$ during the converge process with $p=35$ is the best.
We provide the curves of training loss (green lines) and accuracy (blue lines) in the upper figure of Figure~\ref{fig:training}. When $p$=140, the training process is unstable with an accuracy drop (-0.12\% as shown in Table~\ref{table:reducep}). In addition, the lower figure of Figure~\ref{fig:training} shows the values of weights using norm reduction with different settings. It is obvious that when $p$=35, the curve of weight norm is the most closed to that of using only $\ell_2$-norm, i.e., the network trained using our method can successfully approximate the $\ell_2$-norm wino AdderNet. Thus, we set $p$=35 in our experiments and obtain better performance.

\begin{table}[ht]
\centering
\vspace{-2pt}
\caption{Ablation Study on the Kernel Transformation}
\label{table:kernel}
\begin{tabular}{cc|c}
\hline
\multicolumn{2}{c|}{Method}                                                   & Accuracy \\ \hline
\multicolumn{2}{c|}{Training w/ KT}                                           &    89.19      \\ \cline{1-2}
\multicolumn{1}{c|}{Training} & Init Winograd kernel &     \textbf{91.56}     \\ \cline{2-2}
\multicolumn{1}{c|}{w/o KT}                                 & Init adder kernel and transform    &    91.28      \\ \hline
\end{tabular}
\vspace{-15pt}
\end{table}

We also show the comparison of three ways to deal with kernel transformation. The first way is the same as Winograd algorithm for convolution layers. We apply kernel transformation (KT) to weights during every inference process and update the origin $3\times 3$ kernel. The second way and the third way are to train the network in the Winograd domain. For these two ways, we initialize weights with normal distribution initialization for the $4\times 4$ Winograd kernel and for the $3\times 3$ original Adder kernel, respectively. The results are shown in Table~\ref{table:kernel}. Training with kernel transform has the worst performance, since the inconsistent transform makes the training harder. Other two ways have similar results, so we recommend to directly initialize Winograd kernel due to its convenience.

Next we evaluate the effectiveness of our proposed methods. The results are shown in Table~\ref{table:ablation}. Without any modification, the original Winograd algorithm only achieves 83.87\%. Modified transform matrix and our $\ell_2$-to-$\ell_1$ training strategy brings 4.38\% and 5.38\% accuracy improvement, respectively. Finally the combination achieves 91.56\%, which is comparable with that of AdderNet.

\begin{table}[ht]
\centering
\vspace{-10pt}
\caption{Ablation Study on Our Proposed Methods}
\label{table:ablation}
\begin{tabular}{cc|cc}
\hline
Mod $A$ & $\ell_2$-to-$\ell_1$-norm & CIFAR-10& CIFAR-100\\ \hline
             &                                                                    &   83.87\%      & 54.72\% \\
             &           $\surd$                                                         &  88.25\%    &  62.00\%  \\
     $\surd$         &                                                                    &   89.25\%   &  62.83\%  \\
     $\surd$         &           $\surd$                                                          &   \textbf{91.56\%}   &  \textbf{67.96\%}  \\ \hline
\end{tabular}
\vspace{-15pt}
\end{table}

\subsection{Visualization}
To intuitively perceive the effectiveness of Winograd for AdderNet, we visualize the features. We acquire the features of the last adder layer in our modified LeNet-5-BN, and then use t-SNE method to reduce the dimension to 2. The result of original AdderNet is shown in the left of Figure 3 and the result of Winograd for AdderNet is shown in the right. From the results, we can find that the two are extremely close, which means that Winograd for AdderNet attracts the similar features to original AdderNet.

Besides, we visualize the output features of Winograd Adder and original Adder layer, to intuitively display the effect of modifying matrix $A$. From Figure~\ref{fig:feature}, we can find that output features with original $A$ show a distinct grid while output features with modified $A$ do not show this phenomenon.

\subsection{FPGA simulation}
\label{sec:exp-fpga}
To evaluation the energy efficiency of our method in the runtime, we implement the Winograd algorithm for AdderNet and original AdderNet on FPGA.
The designed parallelism of calculation is 256, which means that 16 input channels and 16 output channels are calculated simultaneously.

We take a single layer with input shape $(N,$$C_{in}, X_h, X_w)=(1, 16, 28, 28)$ and kernel shape $(C_{out}, C_{in}, K_w, K_h)=(16,16,3,3)$ as an example. The comparison of Winograd AdderNet and original AdderNet is shown in Table~\ref{table:fpga}. We can find that Winograd AdderNet requires only $24.0/50.4\approx 47.6\%$ energy consuming of original AdderNet. As we analysed in Section~\ref{sec:basic-form}, the theoretical cost of Winograd AdderNet is $45.4\%$ of that of original AdderNet with $C_{in}=16$ and $C_{out}=16$. So our implementation validates this theoretical result. Moreover, with the pipeline technique, Winograd AdderNet may achieve about 50\% latency reduction (estimated).

\section{Conclusion}
\label{sec:conclusion}
In this paper, we propose the Winograd algorithm for AdderNet. We replace the element-wise multiplications in the Winograd equation with additions to further reduce the energy costs of CNN. To mitigate the accuracy loss brought by the replacement, we develop a set of new transform matrixes to balance output features and introduce an $\ell_2$-to-$\ell_1$ training method for the Winograd AdderNet paradigm. As a result, the proposed method reduces about 52.4\% energy consumption on our FPGA simulation while achieving similar performance with the original AdderNet, which would have a excellent prospects in future hardware design.

\nocite{langley00}

\bibliography{example_paper}
\bibliographystyle{icml2021}

\end{document}